\documentclass[review,10pt]{JMtemplate}
\usepackage{bm}
\usepackage{graphicx} 
\usepackage{amsmath,mathrsfs,amssymb} 
\usepackage{amsfonts}  
\usepackage{algorithm}
\usepackage{algorithmic}
\usepackage{siunitx}   
\usepackage{upgreek}   
\usepackage{xcolor}    
\usepackage{subfigure}
\usepackage{booktabs}      
\usepackage{longtable}
\usepackage{threeparttable}
\usepackage{multirow}
\usepackage{multicol}
\usepackage{times}
\usepackage{helvet}
\usepackage{courier}
\usepackage{marvosym}
\usepackage[hyphens]{url} 
\usepackage{natbib}  
\usepackage{caption} 
\usepackage{adjustbox}
\usepackage{subcaption}
\usepackage{subfigure}
\usepackage{makecell}

\newtheorem{theorem}{Theorem}
\newtheorem{definition}{Definition}

\newenvironment{proof}{{\noindent\it Proof.}\quad}{\hfill $\square$\par}

\newcommand*{\dif}{\mathop{}\!\mathrm{d}}

\newcommand*{\esssup}{\operatorname*{ess\,sup}}

\usepackage[
    colorlinks=true,    
    linkcolor=blue,     
    urlcolor=blue,      
    citecolor=green,    
    filecolor=magenta   
]{hyperref}
\usepackage{textcomp}

\begin{document}
\begin{frontmatter}
\title{Generalization Bounds of Spiking Neural Networks\\ via Rademacher Complexity}

\author{Shao-Qun Zhang\textsuperscript{1,2} \qquad
Zhi-Hua Zhou\textsuperscript{1, 3}}
\address{
\textsuperscript{1} National Key Laboratory for Novel Software Technology, Nanjing University, Nanjing 210063, China.\\
\textsuperscript{2} School of Intelligent Science and Technology, Nanjing University, Suzhou 215163, China.\\
\textsuperscript{3} School of Artificial Intelligence, Nanjing University, Nanjing 210063, China.\\~\\
\texttt{\{zhangsq,zhouzh\}@lamda.nju.edu.cn}
}

\begin{abstract}
Spiking Neural Networks (SNNs) have garnered increasing attention as one of bio-inspired models due to their great potential in neuromorphic computing and sparse computation. Many practical algorithms and techniques have been developed; however, theoretical understandings of the generalization, that is, the extent to which SNNs perform well on unseen data, are far from clear. Recent advances disclosed an excitation-dependent and architecture-related generalization bound such that the Rademacher complexity of SNNs with stochastic firing can be upper bounded by an exponential function relative to the excitation probability and the architecture depth. In this paper, we theoretically investigate the generalization bounds of SNNs with several integration-and-fire schemes via Rademacher complexity. We recognize that the empirical Rademacher complexity of SNNs is close to the SNN configurations, which is exponential to the network depth and the maximum time duration of received spike sequences, superlinear and subquadratic to the network width, polynomial to the parameter norm, inverse-linear to the number of training samples, and independent of the computations within spiking neurons, achieving a more precise rate than conventional studies. Our theoretical results may support the scope of SNN theories and shed some insight into the development of SNNs.

\end{abstract}
\end{frontmatter}

\section{Introduction}  \label{sec:introduction}
In recent years, Spiking Neural Networks (SNNs) have attracted increasing attention due to their potential for event-dependent modeling~\citep{yu2014}, neuromorphic computing~\citep{gerstner2002:GsF}, and sparse computation~\citep{rozell2008sparse}. The SNN building consists of two parts: first, delivering the spike-related information among spiking neurons via connection weights, and second, converting information between membrane potentials and spike sequences within certain spiking neurons. The inside computations of a spiking neuron usually follow the integration-and-fire paradigm. There has been significant progress on computational and implementation techniques for SNNs in computer vision~\citep{quiroga2005}, speech recognition~\citep{verstraeten2005:nlp,schrauwen2008:nlp}, reinforcement learning~\citep{florian2007:reinforcement,vasilaki2009:reinforcement}, few-shot learning~\citep{kheradpisheh2018:few,goltz2021:few}, etc. However, theoretical characterizations, especially the generalization, of SNNs are still far from clear.

The generalization depicts whether and to what extent the studied SNN that has been trained on observed spikes performs well on unobserved spike sequences; perhaps, it is the most fundamental concern in artificial intelligence and neuromorphic computing. Despite the emerging efforts on theoretical characterizations, the generalization of SNNs remains mysterious. \citet{maass1997:vc} investigated the complexity of learning a single spike neuron via the Vapnik-Chervonenkis (VC) dimension~\citep{vapnik2000:VC}. \citet{schmitt1999:vc} further proved VC dimensions of SNNs that maintain the quasi-linear rate with respect to the number of parameters. Since the VC dimension is data-independent, the generalization bound based on VC dimensions appears to be conservative and insensitive to the parameter magnitude and the architecture complexity. \citet{zhang2024:intrinsic} proved an excitation-dependent and architecture-related generalization bound that the Rademacher complexity of SNNs with stochastic firing mechanisms, rather than the conventional integration-and-fire scheme, can be upper bounded by an exponential function relative to the excitation probability and the architecture depth. This bound implies the possibility to efficiently reduce the generalization bound of SNNs by exploiting random algorithms led by stochastic excitation.

\begin{table*}[t]
	\centering
	\caption{Progresses on generalization bounds of SNNs.}
	\vspace{0.1cm}
	\label{tab:generalization_summary}
	\resizebox{\textwidth}{!}{
		\begin{tabular}{l | c | c | l}
			\toprule
			\textbf{Studies}  & \textbf{Configurations} &  \textbf{Approaches}  & \textbf{Generalization Bounds} \\  \midrule
            \citet{maass1997:vc} & \makecell{Single neuron\\ with Temporal Coding} & VC Dimension & VC $\in\Omega(n \log n)$ with regard to $n$ samples\\ \midrule
            \citet{schmitt1999:vc} & \makecell{Binary and Analog Coding} & VC Dimension & \makecell[l]{VC $\in\Omega(LN_w \log (LN_w))$ \\ where $L$ is the network depth.} \\ \midrule
            \citet{neuman2025:covering} & \makecell{Affine Network} & Covering Numbers & $N_\text{nc} \in\mathcal{O}\left( \left( 32N_w^{5/2} + 48LN_wM_w^2 \right) n \exp\left( 3N^2_w \right)  \right)$ \\ \midrule
			\citet{zhang2024:intrinsic}
			& Stochastic firing & \makecell{Rademacher  Complexity\\ Dropout}& \makecell[l]{$\hat{\mathfrak{R}}_n(\mathcal{H}) \in \mathcal{O}\left(  C_{w}^L~ p_{\textrm{max}}^{(L+1)/2}  \right)$ where $p_{\textrm{max}} \in [0,1]$ \\ and $C_w$ universally relates to parameter $w$.} \\  \midrule
			This work & \makecell{Common-used Expressions\\ in Eq.~(\ref{eq:SNN_overall})} & \makecell{Rademacher Complexity\\ Covering Numbers} & $ \hat{\mathfrak{R}}_n(\mathcal{H}) \in \mathcal{O}\left( n^{-1} M_w^L N_w^{3/2} \ T^{L+1} \exp(-TL) \right)$ \\
			\bottomrule
	\end{tabular} }
\end{table*}

In this paper, we theoretically investigate the generalization of SNNs with general spiking neurons. In contrast to the theoretical results of~\citet{zhang2024:intrinsic} that are derived from inside computations of spiking neurons, we here focus on the common-used integration-and-fire schemes and propose a stricter generalization bound of SNNs via the empirical Rademacher complexity. The main result of this work is listed as follows
\begin{theorem} [Generalization Bound for SNNs] \label{thm:DEF}
	Let $\mathcal{H}$ denote the function collection of $L$-layer SNNs with the width of $N_w$ on the time interval $[0,T]$, the hypothesis space $\mathcal{H}$ is universally bounded according to $\|f\|_\infty \leq N_f$ for $f \in \mathcal{H}$, $\mathcal{D}$ is the distribution of sample variables $(\mathbf{X},y)$, $S_n$ indicates the collection of $n$ pairs of training samples $(\mathbf{X}_i,y)$ where $i\in[n]$, and $L_2(S_n)$ denotes the data-dependent $L^2$ metric space. Assume that the norm value of connection weights is finite, that is, $\| \boldsymbol{w} \| \leq M_w$. Let $\hbar:[-N_f,N_f]\times[-N_f,N_f]\to \mathbb{R}$ denote the non-negative loss function, which satisfies that 
	\begin{itemize}
		\item[i)] $\hbar(\cdot,\cdot)$ is upper bounded by $M_\hbar$, i.e., $\hbar( f,f')\le M_\hbar$ for all $f(\cdot),f'(\cdot) \in [-N_f,N_f]$ ,
		\item[ii)] for any fixed $f \in[-N_f,N_f]$, the mapping $y\mapsto \hbar(f,y)$ is $L_\hbar$-Lipschitz for some $L_\hbar>0$ .
	\end{itemize}
	Then with probability at least $1-\delta$ where $\delta \in (0,1)$, we have the generalization bound as follows
	\[
	E(f) \le \hat{E}(f) + 2L_\hbar \, \hat{\mathfrak{R}}_n(\mathcal{H}) + 3M_\hbar \sqrt{\frac{\log(2/\delta)}{2n}} \ ,
	\]
	where $E(f) = \mathbb{E}_{(x,y)\sim\mathcal{D}}[\hbar(f(\boldsymbol{w}, \mathbf{X}),y)]$ denotes the expected error, $\hat{E}(f) = n^{-1}\sum_{i=1}^n \hbar(f(\boldsymbol{w}, \mathbf{X}_i,y_i) )$ denotes the empirical error, and the empirical Rademacher complexity $\hat{\mathfrak{R}}_n(\mathcal{H})$ is bounded by
	\[
	\hat{\mathcal{R}}_n( \mathcal{H} ) \!\leq\!  \frac{128 \,TN_f N_w^{\frac{3}{2}} \log 2}{3\pi n} - 32 \sqrt{ \frac{2 \log 2}{3 \pi}}\sqrt{ \frac{ TN_f^2 N_w^{\frac{3}{2}} }{n}} \ ,
	\]
    where $N_f \in \mathcal{O} [  (T M_w)^L \exp(-TL) ]$.
\end{theorem}
\noindent\textbf{Results.} Theorem~\ref{thm:DEF} shows the generalization bound of SNNs with the typical integration-and-fire scheme via Rademacher complexity. It is observed that the sharp reduction of Rademacher complexity is estimated by a considerably stricter bound than conventional studies. Specifically, the empirical Rademacher complexity we proved is exponential to the network depth $L$ and the maximum time duration $T$ of the received spike sequences, superlinear and subquadratic to the network width $N_w$, polynomial to the parameter norm $M_w$, and inverse linear to the number of training samples $n$, despite the computations within spiking neurons. The experimental results conducted on the delayed-memory XOR task verify our theoretical results.

\noindent\textbf{Comparison.} Table~\ref{tab:generalization_summary} theoretically compares our proposed generalization bound with those based on the VC dimension, Rademacher complexity, and covering numbers. It is observed that the bounds relative to the Rademacher complexity and covering numbers are significantly tighter than those based on the VC dimensions. In contrast to the study of~\citet{neuman2025:covering} that provides a deterministic and worst-case generalization bound via covering numbers, the empirical Rademacher complexity we proved in this work allows for the effects led by stochastic factors, thus supporting the data-dependent and average-case generalization bound. Besides, \citet{zhang2024:intrinsic} focus more on the effects induced by the network depth, the parameter norm, and the excitation probability within spiking neurons, whereas this work leverages more configurations of SNNs, including the time duration, the network depth, the network width, the parameter norm, and the number of training samples. The numerical experiments in Section~\ref{sec:experiments} further demonstrate the influence of these configurations on the generalization performance of SNNs. However, our results are partially consistent with those of~\citet{zhang2024:intrinsic}; the excitation probability of a spiking neuron can be approximately computed by the product of the maximum duration $T$ of the received spike sequence and the empirical firing rate of the whole SNN. This involves an equivalent conversion between the spiking neuron's internal firing mechanism and the external firing rate.

\begin{figure}[t]
	\centering
    \includegraphics[width=0.55\textwidth]{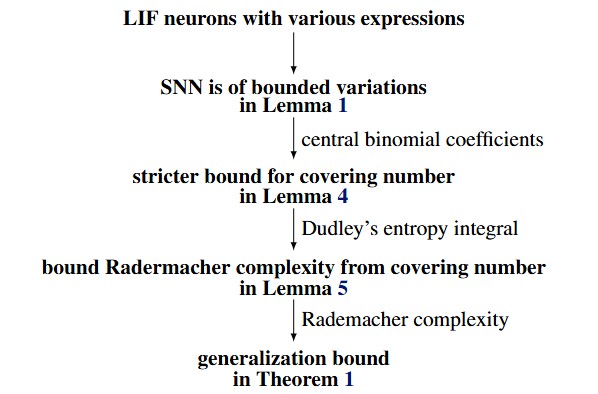}
	\caption{An overview of the proof sketch of Theorem~\ref{thm:DEF}.}
	\label{fig:proof_sketch}
\end{figure}

\noindent\textbf{Proof Sketches.} Figure~\ref{fig:proof_sketch} displays the proof sketches of Theorem~\ref{thm:DEF}. There are generally three key steps: firstly, proving SNN with spiking expressions is of bounded variations in Lemma~\ref{lemma:DEF_is_bv},  secondly, deriving a stricter bound for covering number in Lemma~\ref{lemma:covering_number}, and thirdly, bounding the empirical Rademacher complexity via covering numbers in Lemma~\ref{lemma:RC_CN_01}. Notice that the key idea of proving Theorem~\ref{thm:DEF} is based on the DEF expression that corresponds to the typical integration-and-fire equation and refers to that of~\citet{verma2025generalization}. One can also obtain similar results from other LIF expressions of spiking neurons in Theorem~\ref{thm:others}. 

The proving workflow in Figure~\ref{fig:proof_sketch} also catalyzes the rest of the organization of this paper. Section~\ref{sec:pre} introduces the necessary notations and spiking neural expressions. Section~\ref{sec:proof} completes the proof of Theorem~\ref{thm:DEF} with useful lemmas. Section~\ref{sec:extention} extends the results of Theorem~\ref{thm:DEF} to alternative expressions of spiking neurons. Section~\ref{sec:experiments} conducts experiments to demonstrate the effectiveness of our theoretical results. Section~\ref{sec:rw} reviews the related studies on the theoretical progress of SNNs. Section~\ref{sec:conclusions} concludes this work.

\section{Preliminaries}  \label{sec:pre} 

\subsection{Notations}
Let $[N] = \{1, 2, \dots, N\}$ be an integer set for $N \in \mathbb{N}^+$, and $|\cdot|_{\#}$ denotes the number of elements in a collection, e.g., $|[N]|_{\#} = N$. For $\mathbf{W} \in \mathbb{R}^{n \times m}$, we denote by
\[
\| \mathbf{W} \|_2 \!=\! \left(\! \sum_{i=1}^n \sum_{j=1}^m |\mathbf{W}_{ij}|^2 \!\right)^{1/2}
\text{and}\quad
\| \mathbf{W} \|_\infty \!=\! \max_{i,j} |\mathbf{W}_{ij}| \ .
\]
Here, we only introduce the norms of $\| \cdot \|_2$ and $\| \cdot \|_\infty$. It is evident that the 2-norm can be bounded by the infinity one, i.e., $ \| \boldsymbol{w} \|_2 \leq \sqrt{n} \| \boldsymbol{w} \|_\infty$.

\noindent\textbf{Algorithmic Complexity.} Given two functions $g,h\colon \mathbb{N}^+\rightarrow \mathbb{R}$, we denote by $h\in\Theta(g)$ if there exist positive constants $c_1,c_2$, and $n_0$ such that $c_1g(n) \leq h(n) \leq c_2g(n)$ for every $n \geq n_0$; $h\in\mathcal{O}(g)$ if there exist positive constants $c$ and $n_0$ such that $h(n) \leq cg(n)$ for every $n \geq n_0$; $h\in\Omega(g)$ if there exist positive constants $c$ and $n_0$ such that $h(n) \geq cg(n)$ for every $n \geq n_0$.

\noindent\textbf{Functional Space.} This work describes the expressive power of neural networks in the Sobolev space and functional norm. Let $f_i$ be a scalar function from $K \subseteq \mathbb{R}^N$ to $\mathbb{R}$. Given $\boldsymbol{\alpha} = (\alpha_1, \alpha_2, \dots, \alpha_l)^{\top} \in \mathbb{N}^M$ and $\boldsymbol{x} = (x_1, x_2, \dots, x_N) \in K$, we define 
\[ 
D^{\boldsymbol{\alpha}} f_i(\boldsymbol{x}) = \frac{\partial^{\alpha_1}}{\partial x^{\alpha_1}} \frac{\partial^{\alpha_2}}{\partial x^{\alpha_2}} \dots \frac{\partial^{\alpha_l}}{\partial x^{\alpha_l}} f_i(\boldsymbol{x}) \ .
\]
We define the space of continuous functions $\mathcal{C}^q(K, \mathbb{R})$ for $q \in \mathbb{N}^+$ by a collection of $f_i$, where $f_i \in \mathcal{C}(K, \mathbb{R})$ and $D^r f_i \in \mathcal{C}(K, \mathbb{R})$ for $r \in [q]$. Let $\mu$ be a Lebesgue measure defined on $K$. Further, we define the Lebesgue spaces for the mapping $f: K \to \mathbb{R}^M$, in which 
\[
\left\| f \right\|_{\mathcal{L}_\mu^p(K,\mathbb{R}^M)} \overset{\mathrm{def}}{=} \left(\int_{K} \|f(\boldsymbol{x})\|_2^p \dif \mu(\boldsymbol{x}) \right)^{1/p} < \infty
\quad\text{for}\quad
1\leq p < \infty
\]
and
\[
\left\| f \right\|_{\mathcal{L}_\mu^\infty(K,\mathbb{R}^M)} \overset{\mathrm{def}}{=} \esssup_{\boldsymbol{x} \in K} \|f(\boldsymbol{x})\|_{\infty} < \infty
\quad\text{for}\quad
p=\infty \ ,
\]
where $f \in \mathcal{C}(K,\mathbb{R}^M)$. Let $\mathcal{W}^{q,p}_{\mu}(K,\mathbb{R}^M)$ denote the Sobolev space as the collection of all functions $f \in \mathcal{C}^q(K, \mathbb{R}^M)$ and $D^{\boldsymbol{\alpha}} f \in \mathcal{L}_\mu^p(K,\mathbb{R}^M)$ for all $|\boldsymbol{\alpha}| \in [q]$.

Throughout this paper, we abbreviate the $L^p$ norm as $\|\cdot\|_p$; specifically, it is noted $\| f \|_{\mathcal{L}_\mu^2(K,\mathbb{R}^M)} = \| f \|_2$ and $\| f \|_{\mathcal{L}_\mu^\infty(K,\mathbb{R}^M)} = \| f \|_\infty$ for any $f \in \mathcal{C}(K,\mathbb{R}^M)$. It is evident that $\| f \|_{L_\mu^p(K,\mathbb{R}^M)} \leq \sqrt[p]{\mu(K)} \| f \|_{L_\mu^\infty(K,\mathbb{R}^M)}$ and $\|f\|_2 \leq \sqrt{|K|} \ \|f\|_\infty$ for any $f \in \mathcal{C}(K,\mathbb{R}^M)$. In this paper, we take the $L^2$ norm as the default.

\subsection{Spiking Neural Expressions}  \label{subsec:SNN}
The computational process of SNNs usually follows the integration-and-firing paradigm, which consists of an integration operation and a firing-reset mechanism. The leaky integration-and-firing (LIF) neuron model is the common type of spiking integration operation. In practice, researchers usually employ several expressions for implementing the LIF neuron model, including the Differential Equation Formulation (DEF), Spike Response Model (SRM) scheme~\citep{gerstner1995:srm}, Exponential-Integral Decomposition (EID)~\citep{stein1967:EID}, Discrete-Time Approximation (DTA)~\citep{rotter1999:DTA}, Green’s Function (GsF)~\citep{gerstner2002:GsF}, etc.
	\begin{equation}  \label{eq:SNN_overall}
		\left\{~\begin{aligned}
			\text{DEF:} &\quad \tau_\text{m} \frac{\dif u (t)}{\dif t} = - ( u(t) - u_{\text{rest}}) + \tau_\text{r} f_{\textrm{agg}} (\boldsymbol{x}(t)) \ , \\
			\text{SRM:} &\quad u(t) = \sum_{\text{f}: ~t^\text{f} \leq t} \eta\left( t - t^{\text{f}} \right) 
			+ \sum_{j} w_j \sum_{\text{e}: ~t_j^\text{e} \leq t} \epsilon\left( t - t_j^\text{e} \right) \ , \\
			\text{EID:} &\quad u(t) = u_{\text{rest}} 
			+ \left( u(t_0) - u_{\text{rest}} \right) \exp\left( -\frac{t-t_0}{\tau_{\text{m}}} \right)
			+ \tau_{\text{r}} f_{\textrm{agg}}(\boldsymbol{x}) 
			\left[  1 - \exp\left( -\frac{t-t_0}{\tau_{\text{m}}} \right) \right] \ , \\
			\text{DTA:} &\quad u[t+1] = u[t] + \frac{\Delta t}{\tau_{\text{m}}} 
			\left[ - (u[t]-u_{\text{rest}}) + \tau_{\text{r}} f_{\textrm{agg}}(\boldsymbol{x}[t]) \right]  \ , \\
			\text{GsF:} &\quad u(t) = u_{\text{rest}} + \int_0^t h(t-s) f_{\textrm{agg}}(\boldsymbol{x}(s)) \dif s \quad\text{with}\quad  h(t) = \frac{\tau_{\text{r}}}{\tau_{\text{m}}} \exp\left( \frac{ -t }{ \tau_{\text{m}}} \right) \ .
		\end{aligned} \right.
	\end{equation}
    
Eq.~\eqref{eq:SNN_overall} lists these common-used expressions. Here, $u(t)$ and $u_{\text{rest}}$ separately indicate the membrane and rest potentials of the concerned spiking neuron at timestamp $t$, $\boldsymbol{x}(t) = (x_1(t), \dots, x_m(t))^{\top}$ denotes the $m$-dimensional input signals, both $\tau_\text{m}$ and $\tau_\text{r}$ are positive-valued hyper-parameters with respect to membrane time and membrane resistance, respectively. Here, $f_{\textrm{agg}}$ is an aggregation function, usually with the form of $f_{\textrm{agg}} (\boldsymbol{x}(t)) = \boldsymbol{w}^{\top} \boldsymbol{x}(t)$, where $\boldsymbol{w}$ is the learnable vector of connection weights.

Besides, $\eta(\cdot)$ denotes the reset kernel following a self-spike, $t^\text{f}$ indicates the excitation timing of the concerned spiking neuron, $t_j^\text{e}$ indicates the excitation timing of the spiking neuron $j$, and $\epsilon(\cdot)$ is the post-synaptic potential kernel. In practice, both $\eta(\cdot)$ and $\epsilon(\cdot)$ are typically formulated by exponential decay, Gaussian, and piecewise linear functions, which are usually Lipschitz continuous and have finite derivatives within their domain. Therefore, it is mild to assume that both $\eta(\cdot)$ and $\epsilon(\cdot)$ follow Lipschitz continuity. 

The spiking neuron fires spikes $s(t)$ at time $t$ if and only if $u(t) \geq u_{\text{firing}}$, where $u_{\text{firing}}$ indicates the firing threshold. Here, we employ the spike excitation function to approximate this procedure, that is, $f_e: \mathbb{R} \to \mathbb{R}$, where $s(t) = f_e(u(t)) =  {u(t)} / { u_{\text{firing}} }$. After firing, the membrane potential is instantaneously reset to a lower value $u_{\text{reset}}$, that is, the reset voltage. Formally, one has $u(t) = (1 - s(t)) u(t) + s(t) u_{\text{reset}}$.

In the main text, we take DEF in Eq.~(\ref{eq:SNN_overall}) as the breakthrough point to study the generalization of SNNs and extend the theoretical results of other expressions in the appendix.

\section{Proof of Theorem~\ref{thm:DEF}}  \label{sec:proof}
Figure~\ref{fig:proof_sketch} displays the proof sketches of Theorem~\ref{thm:DEF}, which constructs the organization of this section according to the corresponding steps therein. Before the proof, it is necessary to introduce Gronwall's inequalities, which closely relate to the functions expressed by SNNs with various expressions, and the relation between generalization and Rademacher complexity, with some useful terminologies.

\subsection{Bounded Variation}
We begin our proof by defining the functions of bounded variation or the family of functions of bounded variation, as the functions expressed by SNNs with the aforementioned expressions are observed to exhibit this property. 
\begin{definition}[Functions of Bounded Variation]
	The function $u \in \mathcal{L}^1(\Omega,\mathbb{R})$ is a function of bounded variation on $\Omega$, denoted by $ u \in \text{BV}(\Omega,\mathbb{R})$, if the distributional derivative of $u$ is representable by a finite Radon measure in $\Omega$, i.e.,
	\[
	\int_{\Omega} u \cdot \frac{\partial \varphi}{\partial x_i} \dif x = - \int_{\Omega} \varphi \dif D_i u \ , 
	~ \forall~ \varphi \in \mathcal{C}^1(\Omega,\mathbb{R}) \ , ~ i \in [n] \ ,
	\]
	for some Radon measure $Du = (D_1u, D_2u, \ldots, D_nu)$. We denote by $|Du|$ the total variation of the
	vector measure $Du$
	\[
	|Du|_\Omega = \sup \left\{ \int_{\Omega} u(x) \, \mathrm{div}(\varphi) \dif x \right\} \ ,
	\]
	where $\varphi \in \mathcal{C}^1(\Omega,\mathbb{R}^n), \ \|\varphi\|_{\mathcal{L}^\infty(\Omega)} \le 1$.
\end{definition}

\begin{definition}[Discrete, Function of Bounded Variation]
Let $u \in \mathcal{L}^1([a,b],\mathbb{R})$ be a real-valued function defined on a compact interval $[a,b]$. For any finite partition $\mathcal{P} \stackrel{\mathrm{def}}{=} \{ a = t_0 < t_1 < \cdots < t_n = b \}$, the total variation of $u$ on $[a,b]$ is defined by the supremum  taken over all finite partitions of $[a,b]$, that is,
\[
V_a^b(u) \stackrel{\mathrm{def}}{=} \sup_{\mathcal{P}} \left[ \sum_{i=1}^{n} \left| u(t_i) - u(t_{i-1}) \right| \right] \ .
\]
The function $u$ is said to be of bounded variation on $[a,b]$ if $V_a^b(u) < +\infty$, denoted by $u \in \text{BV}([a,b],\mathbb{R})$. 
\end{definition}
The above definitions refer to~\citet{dutta2018covering}. Next, we have the first key conclusion as follows.
\begin{lemma} \label{lemma:DEF_is_bv}
	In the case of finite spikes in $[0,T]$, the function expressed by an SNN with the DEF scheme is the function of bounded variation. 
\end{lemma}
Lemma~\ref{lemma:DEF_is_bv} shows the well-behaved property of SNNs with the DEF expressions. The proof of Lemma~\ref{lemma:DEF_is_bv} can be accessed from Appendix~\ref{proof:DEF_is_bv}.

\subsection{Stricter Bound of Covering Number}
In this subsection, we are going to connect the function of bounded variants and covering numbers. 

\begin{definition}[Covering Number] \label{def:covering_number}
	Let $(\mathfrak{M},\rho)$ be a metric space. A subset $\mathfrak{I} \subseteq \mathfrak{M}$ is called a $\gamma$-cover of $\mathfrak{I} \subseteq \mathfrak{M}$, if for every $m \in \mathfrak{I}$, there exists an $m' \in \hat{\mathfrak{I}}$ such that $\rho(m,m') \le \gamma$. The $\gamma$-covering number of $\mathfrak{I}$ is defined by $N_{\text{cn}}(\gamma, \mathfrak{I}, \rho) = \min \{ |\hat{\mathfrak{I}}| : \hat{\mathfrak{I}} \text{ is a $\gamma$-cover of $\mathfrak{I}$} \}$.
\end{definition}
The above definition refers to~\citet{bartlett2017:covering}. 

\begin{lemma}[Discrete Gronwall’s Inequality] \label{lemma:gronwall_discrete}
	Let $\{u_k\}_{k\ge0}, \{a_k\}_{k\ge0}, \{b_k\}_{k\ge0}$ be positive sequences of real numbers that  satisfies $u_n \leq a_n + \sum_{l=0}^{n-1} b_l u_l$ for $n\in \mathbb{N}$, then one has
	\[
	u_n \leq a_n + \sum_{l=0}^{n-1} a_l b_l \prod_{j=l+1}^{n-1} (1+b_j) \ , \quad \forall~ n \geq 0 \ .
	\]
\end{lemma}
\begin{lemma}[Continuous Gronwall’s Inequality] \label{lemma:gronwall_continous}
	Let $R_t$ denote the real-valued interval with forms of $[a,\infty)$ or $[a,b]$ where $a<b$, and $\alpha,\beta,u$ are real-valued functions defined on $R_t$. Assume that (1) both $\beta$ and $u$ are continuous, (2) the negative part of function $\alpha$ is integrable on every closed and bounded subinterval of $R_t$, and (3) function $\beta$ is non-negative and $\alpha$ is non-decreasing. If function $u$ satisfies $u(t) \leq \alpha(t) + \int_a^t \beta(s) u(s) \dif s$ for $ t \in R_t$, then one has
	\[
	u(t) \leq \alpha(t) \exp\left( \int_a^t \beta(s) \dif s \right) \ , \quad \forall~ t \in R_t \ .
	\]
\end{lemma}
Lemma~\ref{lemma:gronwall_discrete} and Lemma~\ref{lemma:gronwall_continous} separately provide the discrete~\citep{clark1987:Gronwall} and continuous~\citep{howard2025:Gronwall} versions of Gronwall’s inequalities, which are the key lemmas for connecting the covering number and the function of bounded variants that correspond to SNNs with various expressions. It is intuitive that the discrete Gronwall’s inequality works for the expressions of SRM and DTA, while the continuous Gronwall’s inequality works for those of DEF, EID, and GsF in Eq.~\eqref{eq:SNN_overall}.

Based on the above definition and lemmas, we have the second key lemma as follows.
\begin{lemma} \label{lemma:covering_number}
	Provided two functional indicator sets $\mathfrak{I}_{N_w} = \{ u \in L^1([0,T]^{N_w})\}$ where $|u(t)|$ is a non-decreasing function with respect to time $t$ and $\mathfrak{B}_{N_w} =  \{  u \in L^1([0,T]^{N_w}) \} $ with $|Du|_{(0,T)^{N_w}} \leq M_u \}$,
	the collection of spiking computing functions $ f(\cdot) $ expressed by an $L$-layer SNN with the DEF equation can be upper bounded by 
	\[
	\left\{~ \begin{aligned}
		& N_{\text{cn}}(\gamma,\mathfrak{I}_{N_w}, L_2(S_n)) \leq \left[ \frac{2^{4 TN_f \sqrt{N_w}/\gamma}}{ 6\pi } \right]^{N_w} \ ,\\
		& N_{\text{cn}}(\gamma,\mathfrak{B}_{N_w}, L_2(S_n)) \leq \left[ \frac{2^{16 TN_f \sqrt{N_w} /\gamma}}{ (6\pi)^2 } \right]^{N_w} \ ,
	\end{aligned} \right.
	\]
	where $S_n$ denotes the collection of $n$ pairs of training samples, $L_2(S_n)$ is the data-dependent $L^2$ metric space, and $N_f$ is a universal constant satisfying  $\sup_{f\in\mathcal{H}} \|f\|_\infty \leq N_f$.
\end{lemma}
Lemma~\ref{lemma:covering_number} shows a considerably strict bound for the covering number of the functions expressed by SNNs equipped with DEF neurons. The key to proving Lemma~\ref{lemma:covering_number} comes from an observation that the covering number in Lemma~\ref{lemma:DEF_is_bv} is related to the number of positive integer solutions of the LIF equation, which is equal to central binomial coefficients; the latter obeys a recurrence relation. Thus, it suffices to write the recurrent formation of Eq.~(\ref{eq:SNN_overall}) as
\[
\| f(\boldsymbol{x}(t)) \|_2 \leq A^L  \| \boldsymbol{x}(t) \|_2  + \sum_{l=1}^L B  A^{L-l-1} 
= A^L  \| \boldsymbol{x}(t) \|_2  + \frac{A^{L-1} - A^{-1}}{A-1}  B \stackrel{\mathrm{def}}{=} N_f
\]
with
\[
A = \frac{ t \, \tau_\text{r} }{\tau_\text{m} \, u_{\text{firing}} } \exp\left( \frac{t}{\tau_\text{m}} \right)  \, \|\boldsymbol{w} \|_2 
\quad\text{and}\quad
B = \left[ \frac{ \|u(0)\|_2 }{ u_{\text{firing}} } +  \frac{ t \, \|u_{\text{rest}} \|_2 }{\tau_\text{m} \, u_{\text{firing}} }  \right] \exp\left(  \frac{ t }{\tau_\text{m} }  \right)  \ .
\]
We strictly tighten this bound by exploiting the ratio of gamma functions~\citep{gautschi1959:gamma} as follows
\[
x^{1-\lambda} \le \frac{\Gamma(x+1)}{\Gamma(x+\lambda)} \le (x+1)^{1-\lambda} \ ,
\]
for $x > 0$ and $0<\lambda<1$. The full proof of Lemma~\ref{lemma:covering_number} can be accessed from Appendix~\ref{proof:covering_number}.

\subsection{Generalization and Rademacher Complexity}
In this subsection, we are going to bound the Rademacher complexity by covering numbers. For simplicity, we here focus on typical classification task, such as the delayed-memory XOR~\citep{abbott2016} and the spiking sorter ~\citep{lee2017yass}, of which the output is bounded by $[-N_f,N_f]$ or $-[N_f] \cup \{0\} \cup [N_f]$ where $N_f \in \mathbb{N}^+$. Let $\mathcal{W}$ be the connection weight space for SNNs, and $\mathcal{D}$ denotes the underlying joint distribution over input and output space $\mathcal{X} \times \mathcal{Y}$. The training data set $S_n = \{(\mathbf{X}_i,y_i) \in \mathcal{X} \times \mathcal{Y} \}_{i\in[n]}$ is drawn from $\mathcal{D}$. Thus, we establish the function space as $\mathcal{F}_{\mathcal{W}} = \{ f(\boldsymbol{w}, \mathbf{X} ) \mid \boldsymbol{w} \in \mathcal{W}, \mathbf{X} \in \mathcal{X} \} $. The expected and empirical errors are defined as follows
\[
E(f) = \mathbb{E}_{(\mathbf{X},y)\sim\mathcal{D}} \left[ \hbar \left( f(\boldsymbol{w}, \mathbf{X} ) ,y \right)  \right] 
\quad\text{and}\quad
\hat{E}(f) = \frac{1}{n} \sum_{i=1}^{n} \hbar \left[ f(\boldsymbol{w}, \mathbf{X}_i) ,y_i \right] \ ,
\]
where $\hbar$ denotes the loss function, such as the least square loss and 0-1 loss functions. Here, we mildly assume that the loss function is Lipschitz continuous with respect to $f$, that is,  $| \hbar ( f ,y ) - \hbar ( f' ,y )  | \leq L_\hbar \| f-f' \|_2$ for any $f, f'$ that are usually determined by $\boldsymbol{w}$ and $\mathbf{X}$. The generalization bound describes the gap between $E(f) $ and $\hat{E}(f)$. Rademacher complexity, which measures how well a class of functions can fit random noise, is a typical statistical concept for deriving generalization bounds~\citep{mohri2018:foundations}.
\begin{definition}[Rademacher Complexity]
	Given a class of functions $\mathcal{H}$ mapping from $\mathcal{X}$ to $\mathbb{R}$ and a sample $S_n = \{\mathbf{X}_1,\dots,\mathbf{X}_n\}$ drawn from a distribution $\mathcal{D}$, 	the empirical Rademacher complexity of $\mathcal{H}$ with respect to $S_n$ is
	\[
	\hat{\mathfrak{R}}_n(\mathcal{H}) = \mathbb{E}_{\sigma}\left[ \sup_{f \in \mathcal{H}} \frac{1}{n} \sum_{i=1}^n \sigma_i  f( \mathbf{X}_i ) \right],
	\]
	where $\sigma_i$ are independent Rademacher variables that take values $\pm 1$ with equal probability, and the expectation $\mathbb{E}_\sigma$ is taken over their distribution.
\end{definition}

Next, we are going to introduce the following lemma.
\begin{lemma}   \label{lemma:RC_CN_01}
	Let $\mathcal{H}$ be the hypothesis space, and $L^2(S_n)$ denotes the data-dependent $L^2$ metric space. For $f\in \mathcal{H}$, we have
	\[
	\hat{\mathfrak{R}}_n(\mathcal{H}) \!\le\! \inf_{\epsilon \ge 0} \left\{ \!4\epsilon \!+\! \frac{12}{\sqrt{n}} \!\!\int_{\epsilon}^{\epsilon^+}\!\!\!\!\!\! \sqrt{\log N_{\text{cn}}(\gamma,\mathcal{H},L^2(S_n))} \dif \gamma \!\right\} \ ,
	\]
	where $\epsilon^+ = \sup_{f\in\mathcal{H}} \sqrt{\mathbb{E}[f^2]}$.
\end{lemma}
Lemma~\ref{lemma:RC_CN_01} provides a typical bridge for bounding the Rademacher complexity by covering number. The details can refer to the note of~\citet{srebro2010:covering}.

\noindent\textbf{Finishing the proof of Theorem~\ref{thm:DEF}. } The following inequality holds from Lemma~\ref{lemma:covering_number}
\[
\sqrt{\log N_{\text{cn}}(\gamma,\mathfrak{B}_{N_w}, L_2(S_n)) } \!\leq\! \sqrt{ \frac{8 TN_f N_w^{\frac{3}{2}} \log 2}{3\pi\gamma} } \!\stackrel{\mathrm{def}}{=}\! g(\gamma) \ .
\]
Thus, we have 
\begin{equation}  \label{eq:gamma}
	\int_a^b g(\gamma) \, \dif \gamma = \sqrt{ \frac{32 \log 2}{3 \pi} } \sqrt{ TN_f N_w^\frac{3}{2} } \left[ \sqrt{b} - \sqrt{a} \right] \ ,
\end{equation}
where $\gamma \in [a,b]$. Inserting the above result into Lemma~\ref{lemma:RC_CN_01}, we have
\[
\hat{\mathfrak{R}}_n(\! \mathfrak{B}_{N_w} \!) \!\leq\! \inf_{\epsilon \ge 0} \!\left\{\! 4\epsilon \!\!+\!\! 12 \!\! \int_{\epsilon}^\alpha \!\!\!\!\!\sqrt{ \frac{\log\! N_{\text{cn}}(\gamma, \!\mathfrak{B}_{N_w} \!,\! L^2\!(S_n))}{n} } \!\dif \tau \!\right\} \ .
\]
where $\alpha = \sup_{f \in \mathfrak{B}_{N_w} } \sqrt{\mathbb{E}[f^2]}$. Notice that $\alpha$ can be upper bounded according to the hypothesis norm, where 
\[
\alpha = \sup_{f \in \mathfrak{B}_{N_w} } \!\! \sqrt{\mathbb{E}[f^2]} \leq
\left\{\begin{aligned}
N_f ~\ ,& \text{~if~} \|f\|_\infty \leq N_f \ ,\\
\frac{N_f}{T^{N_f/2}} \ ,& \text{~if~} \|f\|_2 \leq N_f \ ,\\
\end{aligned}\right.
\]
but $\alpha \to +\infty$ if one employs the $\|f\|_2 \leq N_w$. By exploiting Eq.~(\ref{eq:gamma}), it is observed that $\hat{\mathcal{R}}_n( \mathfrak{B}_{N_w}  )$ is less than
\[
\inf_{\epsilon \geq 0} \left\{ 4\epsilon + 32 \sqrt{ \frac{2 \log 2}{3 \pi} }  \sqrt{\frac{TN_f N_w^{3/2} }{n}} \left( \sqrt{b} - \sqrt{\epsilon} \right) \right\} \ .
\]
This implies that
\[
\hat{\mathcal{R}}_n( \mathfrak{B}_{N_w}  ) \!\!\leq\!\!  \frac{128 TN_f N_w^{\frac{3}{2}} \log 2}{3\pi n} \!-\! 32 \sqrt{ \frac{2 \log 2}{3 \pi}}\!\sqrt{ \!\frac{ \alpha TN_f N_w^{\frac{3}{2}} }{n}} \ .
\]

According to the Rademacher complexity regression bounds of~\citet{mohri2018:foundations}, one can bound the  generalization relative to the aforementioned spiking sorter task by Rademacher complexity as follows

\begin{lemma} \label{lemma:rademacher}
	Let $\mathcal{H} = \{ f:\mathcal{X}\to\mathcal{Y}\}$ be a set of functions.  
	Let $\hbar:\mathcal{Y}\times\mathcal{Y}\to\mathbb{R}^+$ be a non-negative loss function upper bounded by $M_\hbar>0$,  and $L_\hbar$-Lipschitz in the second variable.  Then with probability at least $1-\delta$ where $\delta \in (0,1)$, the following holds
	\[
	E(f) \le \hat{E}(f) + 2L_\hbar \, \hat{\mathfrak{R}}_n(\mathcal{H}) + 3M_\hbar \sqrt{\frac{\log(2/\delta)}{2n}} \ .
	\]
\end{lemma}
By inserting the bound of $\hat{\mathcal{R}}_n( \mathfrak{B}_{N_w}  )$ into Lemma~\ref{lemma:rademacher} and $\mathcal{H} \subseteq \mathfrak{B}_{N_w}$, we can finish the proof of Theorem~\ref{thm:DEF}. $\hfill\square$

\section{Extension Results to Other LIF Expressions}  \label{sec:extention}
Based on the main conclusion in Theorem~\ref{thm:DEF} and its proof, we can extend the results to other expressions in Eq.~(\ref{eq:SNN_overall}). According to the Gronwall’s inequalities with discrete~\citep{clark1987:Gronwall} and continuous~\citep{howard2025:Gronwall} versions in Section~\ref{sec:proof}, we divide the LIF expressions in Eq.~(\ref{eq:SNN_overall}) into two categories; the first one that contains EID and GsF employ the discrete Gronwall’s inequality, almost the same as that of DEF, while the second one that comprises SRM and DTA obeys the continuous Gronwall’s inequality. We straightaway list the theoretical results as follows
\begin{theorem} [Generalization Bound for Various LIF Expressions] \label{thm:others}
	Let $T$ is the maximum duration of the received spike sequence, $N_w$ denotes the network width, the hypothesis space $\mathcal{H}$ is universally bounded according to $\|f\|_\infty \leq N_f$ for $f \in \mathcal{H}$, $S_n$ indicates the collection of $n$ pairs of training samples, and $L_2(S_n)$ denotes the data-dependent $L^2$ metric space. Assume that the norm value of connection weights is finite, that is, $\| \boldsymbol{w} \| \leq M_w$. Let $\hbar:[-N_f,N_f]\times[-N_f,N_f]\to \mathbb{R}$ denote the non-negative loss function, which satisfies that 
	\begin{itemize}
		\item[i)] $\hbar(\cdot,\cdot)$ is upper bounded by $M_\hbar$, i.e., $\hbar( f,f')\le M_\hbar$ for all $f(\cdot),f'(\cdot) \in [-N_f,N_f]$ ,
		\item[ii)] for any fixed $f \in[-N_f,N_f]$, the mapping $y\mapsto \hbar(f,y)$ is $L_\hbar$-Lipschitz for some $L_\hbar>0$ .
	\end{itemize}
	Then for each LIF expressions in Eq.~(\ref{eq:SNN_overall}) and the function collection $\mathcal{H}$ of $L$-layer SNNs on the time interval $[0,T]$, with probability at least $1-\delta$ where $\delta \in (0,1)$, the following bound holds
	\[
	E(f) \le \frac{1}{n}\sum_{i=1}^n \hat{E}(f) + 2L_\hbar \, \hat{\mathfrak{R}}_n(\mathcal{H}) + 3M_\hbar \sqrt{\frac{\log(2/\delta)}{2n}} \ ,
	\]
	in which the empirical Rademacher complexity $\hat{\mathfrak{R}}_n(\mathcal{H})$ is bounded by
	\[
	\hat{\mathcal{R}}_n( \mathcal{H} ) \!\leq\! \frac{128 TN_f N_w^{\frac{3}{2}} \log 2}{3\pi n} \!-\! 32 \sqrt{ \frac{2 \log 2}{3 \pi}}\sqrt{\! \frac{ TN_f^2 N_w^{\frac{3}{2}} }{n}}  \ .
	\]
    The upper bound of $N_f$ can be estimated via covering numbers, divided into the following two forms 
    \[
    N_f^{\text{c}} \in \mathcal{O} \left[  (T M_w)^L \exp(-TL) \right]
    \quad \text{and} \quad
    N_f^{\text{d}} \in \mathcal{O} \left[  (T M_w)^{L-1} \exp(-T(L-1)) \right] \ ,
    \]
    where $N_f^{\text{c}}$ provides a uniform upper bound on the function norm relative to the SNNs with the EID and GsF expressions, while $N_f^{\text{d}}$ provides a uniform upper bound on those of the SRM and DTA expressions. 
\end{theorem}
Theorem~\ref{thm:others} reveals that various LIF expressions of spiking neurons do not affect the algorithmic complexity of the SNN generalization bound; only some constants differ. This observation also coincides with the results of Theorem~\ref{thm:DEF}, where the empirical Rademacher complexity is independent of the inside computations within spiking neurons. The proof of Theorem~\ref{thm:others} can be accessed from Appendix~\ref{proof:others}.

\begin{figure*}[t]
	\centering
	\includegraphics[width=1\textwidth]{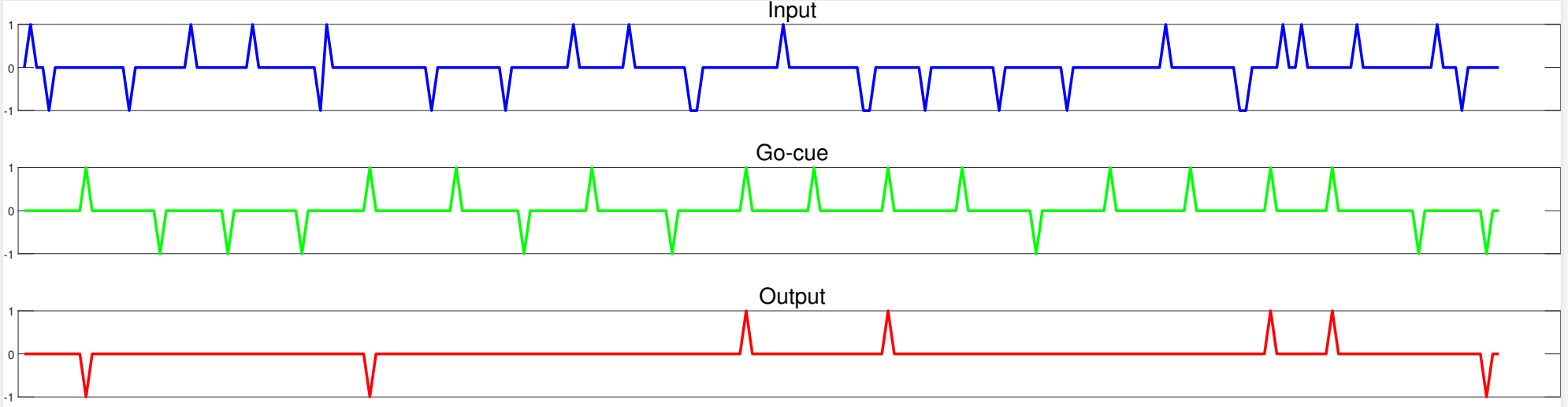}
	\caption{Illustrations of the delayed-memory XOR task, where the panels from top to bottom are the single-trial input, go-cue signals, and output traces, respectively.}
	\label{fig:XOR}
\end{figure*}

\section{Experiments}  \label{sec:experiments}
In this section, we conducted simulation experiments to evaluate the effectiveness of our theoretical results, especially about the influence of the maximum duration $T$ of received spikes sequence, the network width $N_w$, and the network depth $L$ on the generalization performance of SNNs. 

Here, we consider the \emph{delayed-memory XOR} task, which performs the XOR operation on the input history stored over an extended duration~\citep{abbott2016}. Specifically, the network receives two binary pulse signals, + or -, through an input channel and a go-cue channel. When the network receives two input pulses between two go-cue pulses, it should output the XOR signal of both inputs. In other words, the network outputs a positive signal if the input pulses are of equal signs (+ + or - -), and a negative signal if the input pulses are of opposite signs (+ - or - +). If there is only one input pulse between two go-cue pluses, the network should generate a null output.  Here, we simulated a Delayer-memory XOR dataset, which consists of 300 input pulses, 200 go-cue pulses, and the corresponding output signals in the time interval $[0,T]$. We also train SNNs with the rest voltage $u_{\text{rest}}=0$ by the first 80\% timestamps and predict the output signals of the last 20\% timestamps. 

Here, we employ SNNs with the SRM expressions and take values of $T = [500:500:3000]$, $N_w=[2:2:8]$, and $\log_2 L  = [1:1:4]$. The generalization bound is empirically measured by the gap between testing and training errors, that is, $\epsilon = \text{testing error - training error} $. We also test 10 trials for counting the expectation generalization and its variance. All models are implemented via SpikingJelly~\citep{fang2023spikingjelly} and conducted on Intel i9-12900K. 

Figure~\ref{fig:XOR} displays the curves of the generalization performance of SNNs relative to the maximum duration $T$ of the received spike sequence, the network width $N_w$, and the network depth $L$, where the generalization performance is measured by the gap between training and testing errors. There are three observations: (1) There exists a linearly negative correlation between $\epsilon$ and  $\log_2 L$; (2) $N_w$ is negatively correlated with $\epsilon$  at an almost linear rate; (3) the effect of $T$ on $\epsilon$ is not obvious once $T \geq 500$. It is concluded that observation (1) obviously coincides with the results of Theorem~\ref{thm:DEF}, whereas our theoretical analysis is still relatively loose in explaining observations (2) and (3). These results confirm the effectiveness of our theoretical results.

\begin{figure}[t]
	\centering
	\includegraphics[width=1\textwidth]{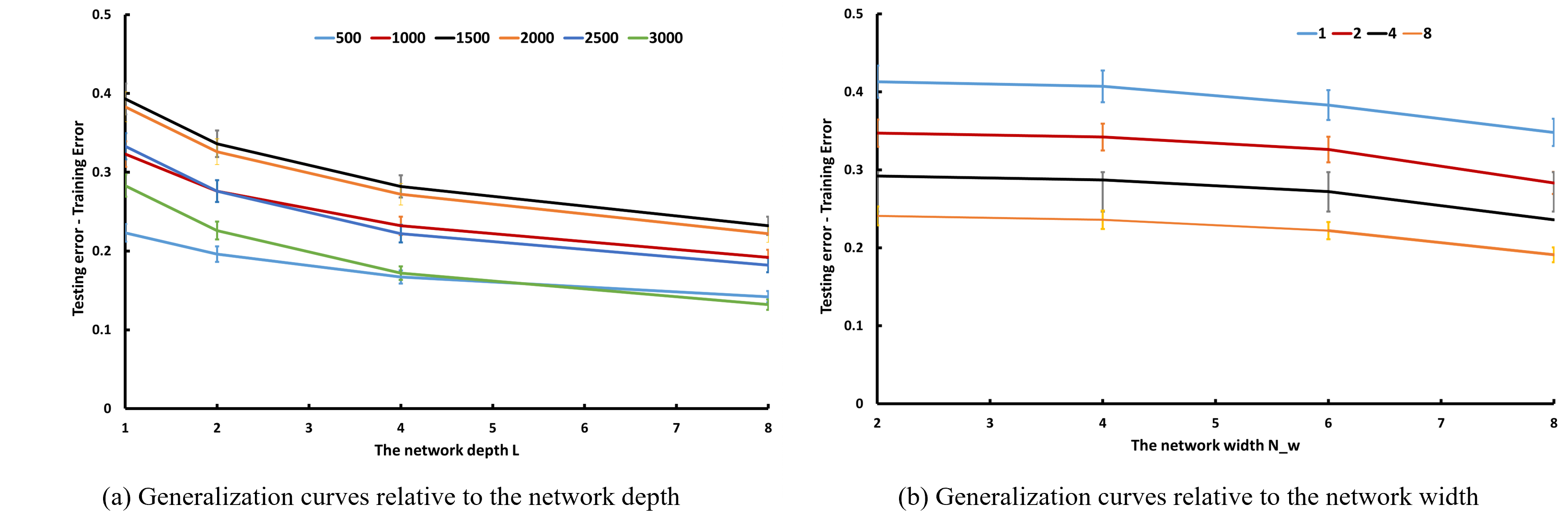}
	\caption{The curves of the generalization performance $\epsilon$ of SNNs relative to (top) the case of the maximum duration $T$ of the received spike sequences and the network depth $L$ provided $N_w =4$ and (bottom) the case of the network width $N_w$ and the network depth $L$ provided $T =2000$.}
	\label{fig:allimages}
\end{figure}

\section{Related Studies}  \label{sec:rw}
There have been several efforts to examine the universality of SNNs, which show that the designed SNNs or spiking neural P systems can simulate some typical computational models involving Turing machines~\citep{maass2001}, random access machines~\citep{maass1997}, threshold circuits~\citep{maass1996lower,maass2004}, propagation paths~\citep{she2021:sequence}, dynamical systems~\citep{zhang2021:bsnn,zhang2022:SNNsTheory}, etc. There are few academic studies on the computational efficiency of SNNs for some specific issues, such as the convergence in the limit results for the sparse coding problem~\citep{tang2016:convergence,tang2017:sparse}, the computational complexity of SNNs for temporal quadratic programming~\citep{barrett2013:firing,chou2018algorithmic}, and the time complexity of approximating multivariate spike flows~\citep{zhang2022:SNNsTheory}. Besides, the firing rates or equally the number of firing spikes are the alternative measure of network activities for investigating neural computation and model dynamics because of the close relation between firing rates and network function, including neural input, connectivity, spiking function, and firing process~\citep{adrian1926:firing_rate,aertsen1980:firing_rate}. The averaged firing rate is used to approximate the optimal solutions of some quadratic programs within polynomial complexity~\citep{chou2018algorithmic}, while the instantaneous firing rate is used to ensure that SNNs can approximate dynamical systems well~\citep{zhang2022:SNNsTheory}. 

Despite the emerging efforts on theoretical characterizations, it remains mysterious whether and to what extent the studied SNN that has been trained on observed spikes performs well on unobserved spike sequences, which is the most fundamental concern in artificial intelligence and neuromorphic computing. A recent advance~\citep{zhang2024:intrinsic} disclosed that the generalization of SNNs with stochastic firing mechanisms, rather than the conventional integration-and-fire scheme, can be upper bounded by an exponential function relative to the excitation probability, which implies a way to reduce the generalization bound of SNNs exponentially by exploiting random algorithms led by stochastic excitation.

\section{Conclusions}  \label{sec:conclusions}
In this paper, we theoretically investigated the generalization of SNNs with the commonly used integration-and-fire schemes and proposed a generalization bound for SNNs based on the empirical Rademacher complexity and covering number.  In contrast to the theoretical results of~\citet{zhang2024:intrinsic} that are derived from stochastic firing mechanisms, our results take a considerably stricter generalization bound and imply that the maximum duration $T$ of the received spike sequence, the network width $N_w$, the network depth $L$, the parameter norm, and the number of training samples intrinsically determine the generalization performance of SNNs despite the inside computations within spiking neurons. Numerical experiments demonstrate the effectiveness of our theoretical results. 

\section*{Acknowledgments}
Shao-Qun Zhang was supported by the Natural Science Foundation of China (62406138).

\appendix
\section*{Appendix}
\section{Proofs of Lemmas relative to Theorem~\ref{thm:DEF}}
This section provides the proofs for three Lemmas~[\ref{lemma:DEF_is_bv}, \ref{lemma:covering_number}, \ref{lemma:RC_CN_01}], which we used to prove Theorem~\ref{thm:DEF}. For convenience, we here take both $\tau_\text{m} $ and $\tau_\text{r}$ as positive values to avoid the redundancy led by $|\tau_\text{m}|$ and $|\tau_\text{r}|$.

\subsection{Proof of Lemma~\ref{lemma:DEF_is_bv}}  \label{proof:DEF_is_bv}
Recall the DEF scheme, that is, 
\[
\tau_\text{m} \frac{\dif u (t)}{\dif t} = - ( u(t) - u_{\text{rest}}) + \tau_\text{r} f_{\textrm{agg}} (\boldsymbol{x}(t)) \ .
\]
For any $t_1, t_2 \in [T]$, we have
\[
|u(t_1) - u(t_2)| \leq \tau_\text{m} \left|  \frac{\dif u (t_1)}{\dif t} - \frac{\dif u (t_2)}{\dif t} \right| + \tau_\text{r} \left| f_{\textrm{agg}} (\boldsymbol{x}(t_1)) - f_{\textrm{agg}} (\boldsymbol{x}(t_2)) \right| \ .
\]
According to the Picard-Lindelof theorem, the membrane potential $u(t)$ in the DEF expression exists uniquely. 

Thus, it is reasonable to conjecture that $\dif u(t) / \dif t$ is bounded. Let $M_{\textrm{agg}}$ denote the maximum norm of the aggregation function $f_{\textrm{agg}} (\cdot)$, that is, $| f_{\textrm{agg}} (\boldsymbol{x}(t)) | \leq M_{\textrm{agg}} \ t$. Thus, one has
\[
\begin{aligned} 
\left| \frac{\dif u (t)}{\dif t} \right| &\leq \frac{1}{\tau_\text{m}} | u(t) | + \frac{1}{\tau_\text{m}} | u_{\text{rest}} | + \frac{\tau_\text{r}}{\tau_\text{m}} \left| f_{\textrm{agg}} (\boldsymbol{x}(t)) \right| \\
&\leq \frac{u_\text{firing} }{\tau_\text{m}} \ t + \frac{u_\text{rest}}{\tau_\text{m}} + \frac{\tau_\text{r}}{\tau_\text{m}} M_{\textrm{agg}} \ t \\
&= \left[ \frac{u_\text{firing} }{\tau_\text{m}} + \frac{\tau_\text{r}}{\tau_\text{m}} M_{\textrm{agg}} \right] t + \frac{ u_\text{rest} }{ \tau_\text{m} } \stackrel{\mathrm{def}}{=} M_u \ t + C \ .
\end{aligned}
\]
Thus, one has
\[
\left|  \frac{\dif u (t_1)}{\dif t} - \frac{\dif u (t_2)}{\dif t} \right|  \leq M_{u}  | t_1 - t_2 |  \ ,
\]
which completes the conjecture.

According to Subsection~\ref{subsec:SNN}, the aggregation function $f_{\textrm{agg}} (\cdot)$ is linear and thus Lipschitz continuous, i.e., there exists a constant $L_{\textrm{agg}}$  such that $| f_{\textrm{agg}} (\boldsymbol{x}(t_1)) - f_{\textrm{agg}} (\boldsymbol{x}(t_2)) | \le L_{\textrm{agg}} | t_1 - t_2 |$. Therefore, we conclude that the membrane potential $u(t)$ in the DEF expression is Lipschitz continuous with a constant $L_u$
\[
|u(t_1) - u(t_2)| \leq \tau_\text{m} M_{u}  | t_1 - t_2 | + \tau_\text{r} L_{\textrm{agg}} | t_1 - t_2 |  \leq L_u | t_1 - t_2 | \ ,
\]
where $L_u = \tau_\text{m} M_{u} + \tau_\text{r} L_{\textrm{agg}}$. 

For any partition $ 0= t_0 < t_1 < \cdots < t_n = T $, one has 
\[
u(t_i) - u(t_{i-1}) = \frac{\dif u(s_i)}{\dif t} \ (t_i - t_{i-1}) \quad \text{for some $s_i \in (t_{i-1}, t_i)$} \ ,
\]
according to the mean value theorem.  By summing up the absolute differences that give the total variation, we have 
\[
\sum_{i=1}^{n} |u(t_i) - u(t_{i-1})| = \sum_{i=1}^{n} \left| \frac{\dif u(s_i)}{\dif t}  \right| (t_i - t_{i-1}) \ .
\]
Thus, the total variation can be bounded by
\[
V_0^T (u) \leq M_u \sum_{i=1}^{n} (t_i - t_{i-1}) = M_u T  \ .
\]
Thus, we can conclude that $u \in \text{BV}([0,T],\mathbb{R})$. Back to the spike excitation function, we can also derive that 
\[
| s(t_1) - s(t_2) |  = \left| f_e(u(t_1))  - f_e(u(t_2)) \right|  \leq \frac{ u(t_1) - u(t_2) }{ u_{\text{firing}} } \ .
\]
It is evident that both $u(t)$ and $s(t)$ have finite total variation due to $M_uT < \infty$ and $M_uT/ u_{\text{firing}} < \infty$. Therefore, the function expressed by single spiking neuron with the DEF expression is of bounded variation. Since connection weights is independent to $t$ and bounded by $M_w$, we can further conclude that the function expressed by an SNN with the DEF expression is the function of bounded variation. This completes the proof. $\hfill\square$

\subsection{Proof of Lemma~\ref{lemma:covering_number}}  \label{proof:covering_number}
We start this proof with the case of $N_w=1$. Let $Du$ denote the distributional derivative of the membrane potential $u(t)$ in the DEF expression and 
\[
\left\{~\begin{aligned}
	& \mathfrak{I} = \{ u \in L^1([0,T]) \mid \text{$u(t)$ is a non-non decreasing function w.r.t. time $t$}\} \ , \\
	& \mathfrak{B} =  \{  u \in L^1([0,T]) \mid  |Du|_{(0,T)}  \leq M_u \} \ .
\end{aligned} \right.
\]
From the conversion of~\citet{zhang2021:bsnn}, the membrane potential $u(t)$ in the DEF expression  
is evidently equivalent to finding a solution to the following equation if $u(t)$ is Lipschitz continuous
\[
u(t) = u(0) + \frac{1}{\tau_\text{m} } \int_0^t - ( u(\tau) - u_{\text{rest}}) + \tau_\text{r} f_{\textrm{agg}} (\boldsymbol{x}(\tau )) \dif \tau \ .
\]
By taking norms, this yields
\[
\begin{aligned}
	\|u(t)\|_2 &\leq \|u(0)\|_2 + \frac{1}{\tau_\text{m} } \int_0^t \left\|  - u(\tau) + u_{\text{rest}} + \tau_\text{r} f_{\textrm{agg}} (\boldsymbol{x}(\tau )) \right\|_2 \dif \tau \\
	&\leq \|u(0)\|_2 + \frac{ 1 }{\tau_\text{m} } \int_0^t \left\| u(\tau) \right\|_2  \dif \tau + \frac{ t }{\tau_\text{m} } \left\|u_{\text{rest}} \right\|_2  + \frac{ \tau_\text{r} }{\tau_\text{m} }  \int_0^t \left\| f_{\textrm{agg}} (\boldsymbol{x}(\tau )) \right\|_2  \dif \tau  \quad \text{( inserting $\int_0^t \dif \tau = t$ )} \\
    &\leq \|u(0)\|_2 + \frac{ 1 }{\tau_\text{m} } \int_0^t \left\| u(\tau) \right\|_2  \dif \tau + \frac{ t }{\tau_\text{m} } \left\|u_{\text{rest}} \right\|_2  + \frac{ t \ \tau_\text{r} }{\tau_\text{m} } \sup_{\tau \in [0,t]} \left\| f_{\textrm{agg}} (\boldsymbol{x}(\tau)) \right\|_2 \ .
\end{aligned}
\]
According to the continuous Gronwall’s inequality in Lemma~\ref{lemma:gronwall_continous}, we have 
\begin{equation} \label{eq:norm_u}
	\begin{aligned}
		\|u(t)\| &\leq \left(  \|u(0)\|_2 + \frac{ t }{\tau_\text{m} }  \|u_{\text{rest}} \|_2 + \frac{ t \ \tau_\text{r} }{\tau_\text{m} } \sup_{\tau \in [0,t]} \left\| f_{\textrm{agg}} (\boldsymbol{x}(\tau)) \right\|_2 \right) \exp\left(  \frac{ 1 }{\tau_\text{m} } \int_0^t \dif \tau \right) \\
		&\leq \left(  \|u(0)\|_2 + \frac{ t }{\tau_\text{m} }  \|u_{\text{rest}}  \|_2 + \frac{ t \ \tau_\text{r} }{\tau_\text{m} } \sup_{\tau \in [0,t]} \left\| f_{\textrm{agg}} (\boldsymbol{x}(\tau)) \right\|_2 \right) \exp\left(  \frac{ t }{\tau_\text{m} }  \right) \ .
	\end{aligned}
\end{equation}
Provided the $L$-layer SNN of the following form\footnote{Here, the superscript indicates the layer. But we omit the superscript of connection weights $\boldsymbol{w}$ for simplicity.}
\[
\left\{ ~\begin{aligned}
	& f(\boldsymbol{x}(t)) = f_e (    u^{(L)}(t)   ) \ , \\
	& \boldsymbol{s}^{(l)}(t) = f_e (    u^{(l)}(t)  ) \quad\text{for}\quad l \in [L] \ , \\
	& u^{(l)} (t) \leftarrow  \text{DEF} \left[ u^{(l)} (t-1) , \boldsymbol{w}^\top \boldsymbol{s}^{(l-1)}(t) \right] \quad\text{for}\quad l \in [L] \ , \\
	& \boldsymbol{s}^{(0)}(t) = \boldsymbol{x}(t) \ , \\
\end{aligned} \right.
\]
the norm of the expressive function can be unfolded as
\[
\begin{aligned}
	\| f(\boldsymbol{x}(t)) \|_2 &= \left\| f_e \left(    u^{(L)}(t)    \right)   - f_e \left(    0   \right) \right\|_2  \leq \frac{ 1 }{ u_{\text{firing}} }  \left\|  u^{(L)}(t)  - 0 \right\|_2  \\
	&\leq \frac{ 1 }{ u_{\text{firing}} }  \left[  \left\|u^{(L)}(0) \right\|_2 + \frac{ t }{\tau_\text{m} } \|u_{\text{rest}}\|_2 + \frac{ t \ \tau_\text{r} }{\tau_\text{m} } \left\| \boldsymbol{w}^\top \right\|_2 \ \sup_{\tau \in [0,t]} \left\| \boldsymbol{s}^{(L-1)}(\tau) \right\|_2  \right] \exp\left(  \frac{ t }{\tau_\text{m} }  \right) \quad \text{( inserting Eq.~(\ref{eq:norm_u}) )}  \\
	&\leq \left[ \frac{ \|u(0)\|_2 }{ u_{\text{firing}} } +  \frac{ t \, \|u_{\text{rest}} \|_2 }{\tau_\text{m} \, u_{\text{firing}} }  \right] \exp\left(  \frac{ t }{\tau_\text{m} }  \right)  + \frac{ t \, \tau_\text{r} }{\tau_\text{m} \, u_{\text{firing}} } \exp\left(  \frac{ t }{\tau_\text{m} }  \right)  \, \|\boldsymbol{w} \|_2  \ \sup_{\tau \in [0,t]} \left\| \boldsymbol{s}^{(L-1)}(\tau)  \right\|_2 \ ,
\end{aligned}
\]
where the last inequality holds upon a mild setting that the initialized membrane potentials of all layers are the same, i.e., $\|u(0)\|_2 = \|u^{(1)}(0)\|_2 = \dots = \|u^{(L)}(0)\|_2$. 

Next, we introduce a useful lemma that relates to Gronwall's inequality~\citep{verma2025generalization}. 
\begin{lemma} \label{lemma:gronwall}
	Let $(u_k)_{k\ge0}$ be a sequence that satisfies $u_k \le a_k u_{k-1} + b_k$ for all $k \ge 1$, where $(a_k)_{k\ge1}, (b_k)_{k\ge1}$ are two positive sequences. Then it holds
	\[
	u_k \le \left(\prod_{j=1}^k a_j\right) u_0 + \sum_{j=1}^k b_j \left(\prod_{i=j+1}^k a_i\right) 
	\quad\text{for all}\quad
	k \ge 1 \ .
	\]
\end{lemma}
According to Lemma~\ref{lemma:gronwall}, we can further bound the norm of the expressive function by
\begin{equation}  \label{eq:norm_f}
\begin{aligned}
\sup_{\tau \in [0,t]} \| f(\boldsymbol{x}(\tau)) \|_2 
&\leq A^L(\tau)  \sup_{\tau \in [0,t]} \| \boldsymbol{x}(\tau) \|_2  + \sum_{l=1}^L B(\tau)  A^{L-l-1}(\tau)  \\
&= A^L(\tau) \sup_{\tau \in [0,t]} \| \boldsymbol{x}(\tau) \|_2  + \frac{A^{L-1}(\tau) - A^{-1}(\tau)}{A(\tau)-1}  B(\tau) \ ,     
\end{aligned}
\end{equation}
where
\[
A(t) = \frac{ t \, \tau_\text{r} }{\tau_\text{m} \, u_{\text{firing}} } \exp\left( \frac{t}{\tau_\text{m}} \right)  \|\boldsymbol{w} \|_2 
\quad\text{and}\quad
B(t) = \left[ \frac{ \|u(0)\|_2 }{ u_{\text{firing}} } +  \frac{ t \, \|u_{\text{rest}} \|_2 }{\tau_\text{m} \, u_{\text{firing}} }  \right] \exp\left(  \frac{ t }{\tau_\text{m} }  \right)  \ .
\]
According to 
\[
\| f(\boldsymbol{x}(t)) \|_2 \leq \sup_{t \in [0,T]} \| f(\boldsymbol{x}(t)) \|_2 \leq \|f\|_\infty \leq N_f
\quad\text{and}\quad
\|f\|_2 \leq \sqrt{|K|} \ \|f\|_\infty \leq N_f \ ,
\]
we can employ $N_f$ to upper bound $\sup_{t \in [0,T]} \| f(\boldsymbol{x}(t)) \|_2$. Provided that $\|\boldsymbol{w}\|_2 \leq M_w$ and $\|\boldsymbol{x}(t)\|_2 \leq M_x \approx 1$, we can intuitively force that
\[
N_f \stackrel{\mathrm{def}}{=}  \tilde{A}^L  M_x  + \frac{\tilde{A}^L - 1}{ \tilde{A}(\tilde{A}-1)}  \tilde{B} 
\]
with 
\[
\tilde{A} = \frac{ T \, \tau_\text{r} }{\tau_\text{m} \, u_{\text{firing}} } \exp\left( \frac{ T }{\tau_\text{m}} \right)  M_w
\quad\text{and}\quad
\tilde{B} = \left[ \frac{ \|u(0)\|_2 }{ u_{\text{firing}} } +  \frac{ T \, \|u_{\text{rest}} \|_2 }{\tau_\text{m} \, u_{\text{firing}} }  \right] \exp\left(  \frac{ T }{\tau_\text{m} }  \right)  \ .
\]
Therefore, we can conclude that $N_f \in \mathcal{O} [  (T M_w)^L \exp(-TL) ]$ from which
\begin{itemize}
    \item[(1)] $\max_T ~ N_f(T,L) \in \mathcal{O} [\ \exp(-L) \ ] $,
    \item[(2)] $N_f(T,L) \to 0$ as $T \to 0^+$ or $T \to + \infty$,
    \item[(3)] $N_f(T,L) \to 0$ as $T \to + \infty$ with an exponential ratio.
\end{itemize}

Next, we proceed to compute $N_{\text{cn}}(\gamma , \mathfrak{I}, L_2(S_n)) $. The proof line follows that of~\citet{verma2025generalization}. For a fixed positive integer $N$, let us set the discretization size as $\Delta x = T/N$, $\Delta y = {N_f}/{N}$. To each $z \in \mathfrak{I}$, we associate the pair of functions $(\psi^+[z], \psi^-[z])$ defined by
\[
\psi^+[z] = \sum_{k=0}^{N-1} \psi^+_k \cdot \mathbf{1}[k \cdot \Delta x, (k+1)\cdot \Delta x] \ ,
\]
where
\[
\psi^-_k = \left\lfloor \frac{z(k \cdot \Delta x + 0)}{\Delta y} \right\rfloor
\quad \text{and} \quad
\psi^+_k = \left\lfloor \frac{z((k+1)\cdot \Delta x - 0)}{\Delta y} \right\rfloor + 1 \ .
\]
For $\mathcal{X}^- \leq \mathcal{X}^+ \in \mathfrak{I}$, one defines $U(\mathcal{X}^-, \mathcal{X}^+) = \{z \in \mathfrak{I} \mid \mathcal{X}^- \leq z \leq \mathcal{X}^+\}$. It is easily proved that the set $\mathcal{U} = \{ U(\mathcal{X}^-[z], \mathcal{X}^+[z]) \mid f \in \mathcal{I} \}$ is a covering of $ \mathfrak{I}$ due to $z \in U(\mathcal{X}^-[z], \mathcal{X}^+[z])$.

According to
\[
\begin{aligned}
	\#\mathcal{U} &\leq \{ 0 \leq a_0 \leq a_1 \leq \cdots \leq a_{N-1} \leq N \mid (a_k \in \mathbb{N})\}^2  \\
	&\leq \{(p_1,\dots,p_{N+1}) \in \mathbb{N}^{N+1} \mid p_1+\cdots+p_{N+1}=N\} ^2 \\
	&\leq \binom{2N}{N} ^2 \ ,
\end{aligned}
\]
the covering number for the class of functions in $\mathfrak{I}$ is bounded by $\binom{2N}{N}^2$. Consider sums of powers of binomial coefficients
\[
a_n^{(r)} = \sum_{k=0}^n \binom{n}{k}^r \ .
\]
For $r=2$, the closed-form solution is given by
\[
a_n^{(2)} = \binom{2n}{n} \ .
\]
This implies that the central binomial coefficients $a_n^{(2)}$ obeys the recurrence relation, that is,
\[
(n+1)a_{n+1}^{(2)} - (4n+2)a_n^{(2)} = 0 \ .
\]
By solving the aforementioned upper bound of $\#\mathcal{U}$, we have
\[
\begin{aligned}
	\binom{2N}{N} &= C_1 \frac{4^{N-1}}{\Gamma(N+1)} \left(\frac{3}{2}\right)_{2N-1} \quad \text{( here, $((x))_N$ denotes the Pochhammer symbol )}  \\
	& = 2 \cdot \frac{2^{2(N-1)}}{\Gamma(N+1)} \left(\frac{3}{2}\right)_{2N-1} \quad \text{ ( let $C_1=2$ ) }  \\
	&= \frac{2^{2(N-1)}}{\Gamma(N+1)}  \frac{\Gamma(\tfrac{3}{2}+n-1)}{\Gamma(\tfrac{3}{2})}  
	= \frac{2^{2(N-1)}}{\Gamma(N+1)} \Gamma\!\left(N+\frac{1}{2}\right) \frac{\sqrt{\pi}}{2}  \\
	& = \frac{2^{2N}}{\sqrt{\pi}} \frac{\Gamma(N+\frac{1}{2})}{\Gamma(N+1)}  \\
	& \leq \frac{2^{2N}}{\sqrt{\pi}}  \frac{1}{\sqrt{N}} \quad \text{( from the ratio of gamma functions~\citep{gautschi1959:gamma} )} \\
	&= \frac{2^{2N}}{\sqrt{\pi N}} \ .
\end{aligned}
\]
Hence, we can conclude that
\[
\binom{2N}{N}^2 \leq \frac{2^{4N}}{\pi N}
\leq \frac{2^{4N}}{6\pi} \ ,
\]
where the second inequality holds from $N\geq 6$. Let $N = \left\lceil {(T N_f)}/{\gamma} \right\rceil + 1$, then
\[
N_{\text{cn}}(\gamma,\mathfrak{I}, L_2(S_n)) \leq \frac{2^{4 (TN_f) /\gamma}}{ 6\pi } \ .
\]
From the bound proposed by~\cite{dutta2018covering}, that is, 
\[
N_{\text{cn}} \left( \gamma,\mathfrak{B},L^2(S_n) \right) \le N_{\text{cn}}^2 \left(\gamma/2,\mathfrak{I},L^2(S_n) \right) \ ,
\]
a stricter bound is proved by
\[
N_{\text{cn}}(\gamma,\mathfrak{B}, L_2(S_n)) \leq \frac{2^{16 (TN_f) /\gamma}}{ (6\pi)^2 } \ .
\]
The above computations can be easily extended to the case of $N_w \geq 1$ where all variables are still bounded by vector or matrix norms. For $u(t) \in \mathbb{R}^{N_w} $, we have
\[
\left\{~ \begin{aligned}
	& N_{\text{cn}}(\gamma,\mathfrak{I}_{N_w}, L_2(S_n)) \leq \left[ \frac{2^{4 (TN_f) \sqrt{N_w}/\gamma}}{ 6\pi } \right]^{N_w} \ ,\\
	& N_{\text{cn}}(\gamma,\mathfrak{B}_{N_w}, L_2(S_n)) \leq \left[ \frac{2^{16 (TN_f) \sqrt{N_w} /\gamma}}{ (6\pi)^2 } \right]^{N_w} \ .
\end{aligned} \right.
\]
This completes the proof. $\hfill\square$

\section{Proofs and Useful Lemmas of Theorem~\ref{thm:others}}  \label{proof:others}
This section provides the proofs for Theorem~\ref{thm:others}. The results of other LIF expressions follow the proof of Theorem~\ref{thm:DEF}. Hence, we here only show the computational difference led by the SRM and DTA expressions. We begin the proof by taking the SRM expression as an example.

\begin{lemma} \label{lemma:SRM_is_bv}
	In the case of finite spikes in $[0,T]$, the function expressed by an SNN with the SRM scheme is the function of bounded variation. 
\end{lemma}
\begin{proof}
	Recall the SRM scheme, that is, 
	\[
	u(t) = \sum_{\text{f}: ~t^\text{f} \leq t} \eta\left(t - t^{\text{f}} \right)  + \sum_{j} w_j \sum_{\text{e}: ~t_j^\text{e} \leq t} \epsilon\left( t - t_j^\text{e} \right) \ .
	\]
	According to Subsection~\ref{subsec:SNN}, the kernels $\eta(\cdot)$ and $\epsilon(\cdot)$ are Lipschitz continuous, i.e., there exist constants $L_\eta$ and $L_\epsilon$ such that
	\[
	|\eta(t_1) - \eta(t_2)| \le L_\eta |t_1 - t_2|
	\quad \text{and}\quad
	|\epsilon(t_1) - \epsilon(t_2)| \le L_\epsilon |t_1 - t_2| \ .
	\]
	For any $t_1, t_2 \in [T]$, we have
	\[
	|u(t_1) - u(t_2)| \leq \sum_{\text{f}: ~t^\text{f} < t} \left| \eta\left(t_1 - t^f\right) - \eta\left(t_2 - t^f\right) \right|  + \sum_j \sum_{\text{e}: ~t_j^\text{e} < t} \left| \epsilon \left(t_1 - t_j^\text{e} \right) - \epsilon \left(t_2 - t_j^\text{e} \right) \right| \ .
	\]
	Consider a finite number of spikes $N_f$ and $N_e$, the above inequality can be written by 
	\[
	|u(t_1) - u(t_2)| \le N_f L_\eta |t_1 - t_2| + N_e L_\epsilon |t_1 - t_2| = L_u |t_1 - t_2| \ , 
	\]
	where $L = N_f L_\eta + N_e L_\epsilon$. Thus, we can conclude that the SRM function is Lipschitz continuous. 
	
	For any partition $ 0= t_0 < t_1 < \cdots < t_n = T $, one has 
    \[
    u(t_i) - u(t_{i-1}) = \frac{\dif u(s_i)}{\dif t} \ (t_i - t_{i-1}) \quad \text{for some $s_i \in (t_{i-1}, t_i)$} \ ,
    \]
	according to the mean value theorem.  By summing up the absolute differences that gives the total variation, we have 
	\[
    \sum_{i=1}^{n} |u(t_i) - u(t_{i-1})| = \sum_{i=1}^{n} \left| \frac{\dif u(s_i)}{\dif t}  \right| (t_i - t_{i-1}) \ .
    \]
	Since $| D u(t) | \leq M_u$, the total variation can be bounded by
	\[
	V_0^T (u) \leq M_u \sum_{i=1}^{n} (t_i - t_{i-1}) = M_u T  \ .
	\]
Thus, it holds $u \in \text{BV}([0,T],\mathbb{R})$. Provided the spike excitation function, we can conclude that 
	\[
	| s(t_1) - s(t_2) |  = \left| f_e(u(t_1))  - f_e(u(t_2)) \right|  \leq \frac{ u(t_1) - u(t_2) }{ u_{\text{firing}} } \ .
	\]
	It is evident that both $u(t)$ and $s(t)$ have finite total variation due to $M_uT < \infty$ and $M_uT/ u_{\text{firing}} < \infty$. Therefore, the function expressed by a single spiking neuron with the SRM scheme is of bounded variation. Since connection weights are independent of $t$ and bounded by $M_w$, we can further conclude that the function expressed by an SNN with the DEF scheme is a function of bounded variation. This completes the proof. 
\end{proof}

\bibliographystyle{abbrvnat}
\bibliography{refs}

\end{document}